\documentclass[10pt,letterpaper,twocolumn]{article}

\usepackage[margin=0.75in]{geometry}
\usepackage{amssymb}
\usepackage{epstopdf}
\usepackage{amsthm}
\usepackage{lmodern}
\usepackage{times}
\usepackage[dvipsnames, svgnames, x11names, hyperref]{xcolor}
\usepackage{float}
\usepackage{subfig}
\usepackage{graphicx}
\usepackage{booktabs}
\usepackage{balance}

\usepackage[bookmarks=true, colorlinks]{hyperref}
\hypersetup{linkcolor={NavyBlue},citecolor={NavyBlue},urlcolor={NavyBlue}}
\usepackage[labelfont={bf},font={small},labelsep=period]{caption}
\usepackage{apacite}

\usepackage{amsfonts}
\usepackage{tikz}
\usepackage{xcolor}
\usepackage{amsmath}
\usepackage{amsthm}

 \newtheorem{theorem}{Theorem}
  
 \newtheorem{Proposition}[theorem]{Proposition}

\DeclareMathOperator*{\argmin}{argmin}
\DeclareMathOperator*{\argmax}{argmax} 
 
\DeclareMathOperator*{\E}{\mathbb{E}}

\newcommand{\reals}{\mathbb{R}}

\newcommand{\simplex}[1]{\bigtriangleup(#1)}

\newcommand{\U}{\mathcal{U}}
\newcommand{\F}{\mathcal{F}}
\newcommand{\G}{\mathcal{G}}
\newcommand{\M}{\mathcal{M}}

\renewcommand{\L}{\mathcal{L}}

\newcommand{\eqn}[1]{\begin{equation} #1 \end{equation}}
\newcommand{\eq}[1]{\begin{equation*} #1 \end{equation*}}
\newcommand{\eqarrn}[1]{\begin{eqnarray} #1 \end{eqnarray}}

\graphicspath{{figs/}}

\usepackage{cleveref}
\usepackage[bottom]{footmisc}

\newcommand{\figref}[1]{{Figure}~\ref{#1}}

\def\thmAMRSA{1}
\def\thmRDRSA{2}

\newcommand{\Fa}{\F_{\alpha}}
\newcommand{\Ga}{\G_{\alpha}}
\newcommand{\Ep}[1]{\mathbb{E}_{#1}}
\newcommand{\dpp}[2][]{\frac{\partial #1}{\partial #2}}
\newcommand{\Linf}{L^*}
\newcommand{\Sinf}{S^*}

\newcommand{\La}{L^*_{\alpha}}
\newcommand{\Sa}{S^*_{\alpha}}
\newcommand{\Qa}{Q_{\alpha}}
\newcommand{\Za}{Z_{\alpha}}
\renewcommand{\S}{\mathcal{S}}

\title{ \bf  A Rate--Distortion view of human pragmatic reasoning}

\author{
  {\bf Noga Zaslavsky}$^{1,2}$
  \and
  {\bf Jennifer Hu}$^1$
  \and
  {\bf Roger P. Levy}$^1$
}

\date{  $^1$Department of Brain and Cognitive Sciences, $^2$Center for Brains Minds and Machines\\
  Massachusetts Institute of Technology\\
  \texttt{\{nogazs, jennhu, rplevy\}@mit.edu}}

\begin{document}

\maketitle

\begin{quote}
\small
\textbf{\normalsize Abstract.~}~What computational principles underlie human pragmatic reasoning? A prominent approach to pragmatics is the Rational Speech Act (RSA) framework, which formulates pragmatic reasoning as probabilistic speakers and listeners recursively reasoning about each other.
While RSA enjoys broad empirical support, it is not yet clear whether the dynamics of such recursive reasoning may be governed by a general optimization principle.
Here, we present a novel analysis of the RSA framework that addresses this question. First, we show that RSA recursion implements an alternating maximization for optimizing a tradeoff between expected utility and communicative effort. On that basis, we study the dynamics of RSA recursion and disconfirm the conjecture that expected utility is guaranteed to improve with recursion depth. Second, we show that RSA can be grounded in Rate--Distortion theory, while maintaining a similar ability to account for human behavior and avoiding a bias of RSA toward random utterance production. This work furthers the mathematical understanding of RSA models, and suggests that general information-theoretic principles may give rise to human pragmatic reasoning.
\end{quote}

\section{Introduction}

The ability to reason about meaning in the local context of social interactions is a fundamental aspect of human language. For example, pragmatic reasoning is key to inferring non-literal meanings of utterances, efficient use of linguistic ambiguity~\cite{peloquin-etal:2019}, and driving language evolution and change~\cite{Sperber2010,Traugott2012}. Thus, understanding the computational principles that may give rise to pragmatic reasoning is important for studying the forces that shape language and cognition. 

A prominent computational approach to pragmatics is the Rational Speech Act (RSA) framework \cite{Frank2012,Goodman2016}. RSA formulates pragmatic reasoning as probabilistic speakers and listeners recursively reasoning about each other's state of mind with the goal of cooperatively gaining communicative utility.
This framework enjoys broad empirical support across a number of psycholinguistic phenomena, such as scalar implicature, irony, metaphor, and hyperbole (for review:~\citeNP{Goodman2016}). 
In many cases, RSA's recursive reasoning is assumed to terminate after one or two iterations, although
several studies have explored and motivated deeper recursions~(e.g., ~\citeNP{Camerer2004,Franke2016,bergen-levy-goodman:2016-sp,Levy2018,Frank2018}).
However, much remains unknown about the dynamics of RSA recursion and whether it can be characterized by a well-motivated optimization principle. It has been conjectured that RSA dynamics is guaranteed to increase expected utility (e.g.,~\citeNP{yuan2018,peloquin-etal:2019}), but these explorations have relied on numeric simulations, leaving open key questions about the dynamics of RSA models.

In this work we present a set of analytic results, demonstrated by model simulations, that extend the theoretical understanding of the RSA framework and ground it in Rate--Distortion (RD) theory~\cite{Shannon1948} --- the subfield of information theory that characterizes efficient compression under limited resources, and has recently been applied to other aspects of language~\cite{Zaslavsky2018,Hahn2020} and cognition~\cite{Sims2016,Sims2018,Gershman2020}.
Specifically, our main contributions are:
(1) We show that RSA recursion is an instance of the alternating maximization algorithm~\cite{Csiszar2004}. However, this optimization does not maximize expected utility as previously conjectured, but rather a tradeoff between expected utility and communicative effort.
(2) We show that the RSA tradeoff can be generalized to a type of RD tradeoff, which yields a slightly modified model of pragmatic reasoning. We refer to this model as RD-RSA.
(3) Building on these results, we study the dynamics of RSA and RD-RSA and compare their predictions. We find that RD-RSA is similar to RSA in its ability to account for human behavior, while avoiding a bias of RSA toward non-informative random utterance production.
Taken together, these results suggest that human pragmatic reasoning may be understood in terms of RD theory.

\subsection*{The Rational Speech Act framework}

We begin by reviewing the formal setup of the RSA framework, which forms the basis for our analysis. The RSA framework provides a class of models of pragmatic reasoning, based on a reference game involving a speaker and a listener.
The game is defined by a set of meanings (referents) $\M$, a set of utterances $\U$, and a lexicon $l(m,u)$ that determines the literal meanings of each utterance (see \figref{fig:lexicon} for example). Given a meaning $m\in\M$ drawn from a prior distribution $P(m)$, the speaker communicates $m$ to the listener by producing an utterance $u\in\U$ according to a production distribution $S(u|m)$. 
Upon hearing this utterance, the listener interprets it according to an inference distribution $L(m|u)$.
RSA recursively relates the speaker and listener by assuming that the speaker rationalizes about the listener's inferences, and that the listener is Bayesian with respect to the speaker's distribution.
This recursion is typically initialized with a literal listener $L_0$ that makes inferences based on the literal meaning of utterances; that is, 
$L_0(m|u) \propto l(m,u)P(m)$.
The pragmatic speaker is bounded-rational with respect to a utility function $V_L(m,u)$, typically defined by
\eqn{
V_L (m,u) =  \log L(m|u) - C(u)\,, \label{eq:V}
}
where $C(u)\ge0$ specifies the cost of $u$. The speaker and listener at recursion depth $t\ge1$ are defined by
\eqarrn{
S_t(u|m) &\propto& e^{\alpha V_{t-1}(u,m)} \label{eq:S}\\
L_t(m|u) &\propto& S_t(u|m)P(m)\,, \label{eq:L}
}
where $V_{t-1}$ is simplified notation for $V_{L_{t-1}}$, and $\alpha\ge0$ controls the degree to which the speaker is rational with respect to maximizing utility.
The dynamics of this recursive process are illustrated in~\figref{fig:rsa-dynamics}.
We denote by $\Sinf$ and $\Linf$ the speaker's and listener's distributions at the limit $t\rightarrow\infty$.

\begin{figure}[t!]
    \centering
    \subfloat[]{
    \renewcommand{\arraystretch}{1.1}
    \begin{tabular}{r|ccc}
     & \begin{tabular}[c]{@{}c@{}}\textsc{m}\\\includegraphics[height=0.8cm]{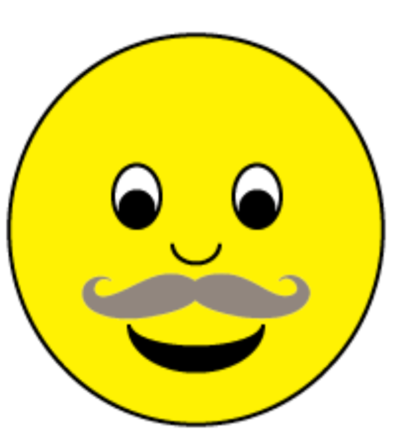}\end{tabular} & \begin{tabular}[c]{@{}c@{}}\textsc{gm}\\\includegraphics[height=0.8cm]{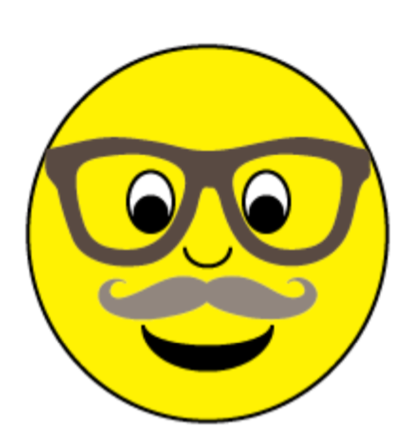}\end{tabular} & \begin{tabular}[c]{@{}c@{}}\textsc{hg}\\\includegraphics[height=0.8cm]{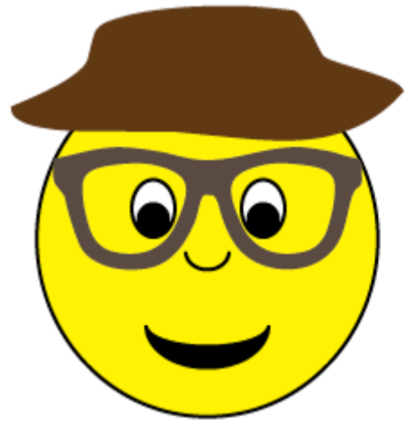}\end{tabular} \\ \hline
    mustache & 1 & 1 & 0 \\
    glasses & 0 & 1 & 1 \\
    hat & 0 & 0 & 1
    \end{tabular}
    \label{fig:lexicon}
    }\qquad
    \subfloat[]{
    \includegraphics[width=.15\linewidth]{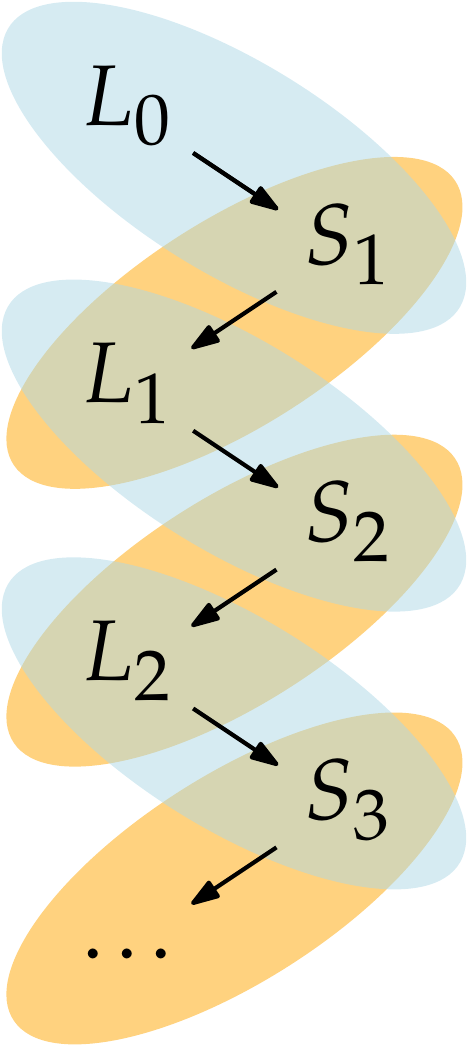}
    \label{fig:rsa-dynamics}
    }
    \caption{\textbf{(a)} Example lexicon with meanings (columns) and utterances (rows) adapted from \protect\citeA{Vogel2014}.
    Upon hearing \emph{glasses}, a pragmatic listener must make a multi-step inference to select \textsc{gm}: while both \textsc{gm} and \textsc{hg} have another possible descriptor, \emph{hat} could have unambiguously identified \textsc{hg}, whereas \emph{mustache} ambiguously refers to both \textsc{m} and \textsc{gm}.
    \textbf{(b)}~Illustration of the dynamics of RSA recursion. 
    The literal listener $L_0$ is initialized from the lexicon. Black arrows correspond to agent updates.
    Blue ovals highlight $(L_t, S_{t+1})$ agent pairs, and orange ovals highlight $(S_t, L_t)$ agent pairs.
    }
     \label{fig:rsa-setup}
\end{figure}

\section{Understanding RSA dynamics as\\ Alternating--Maximization}
\label{sec:AM-RSA}

Our first theoretical result is that the RSA recursion optimizes a tradeoff between maximizing expected utility,
\eqn{
\Ep{S}[V_L] = \sum_{m,u} P(m)S(u|m)V_L(m,u)\,,
}
and maximizing the conditional entropy of the speaker's production distribution,
\eqn{
H_S(U|M) = -\sum_{m,u} P(m) S(u|m) \log S(u|m)\,.
}
More precisely, we show that for any $\alpha\ge0$, RSA's pragmatic interlocutors jointly maximize
\eqn{
\Ga[S, L] = H_S(U|M) +\alpha\Ep{S}[V_L]\,.\label{eq:MaxEnt-RSA}
}
In this view, $\alpha$ does not trade off against recursion depth, as is widely understood~\cite{Frank2018}. Rather, the value of $\alpha$ determines the tradeoff between one conception of communicative effort, $H_S(U|M)$, and expected utility that is optimized by RSA recursion.

Importantly, this result holds without assuming that $S$ and $L$ satisfy equations~\eqref{eq:S} and~\eqref{eq:L} respectively. Instead, the optimization is taken over all valid speaker and listener distributions, and the RSA equations emerge as characterizing the optimal agents.
The main idea of the proof is to take the derivatives of $\Ga$ with respect to $S$ and $L$, and equate these derivatives to zero to obtain conditions for optimality. 
The formal statement and proof of this claim are given in the Supplementary Material (SM, Proposition~\thmAMRSA).
Here, we discuss the implications of this result and demonstrate it numerically. 

\begin{figure}[t!]
\centering
\includegraphics[width=.85\linewidth, trim={2em 1em .1em 1em}]{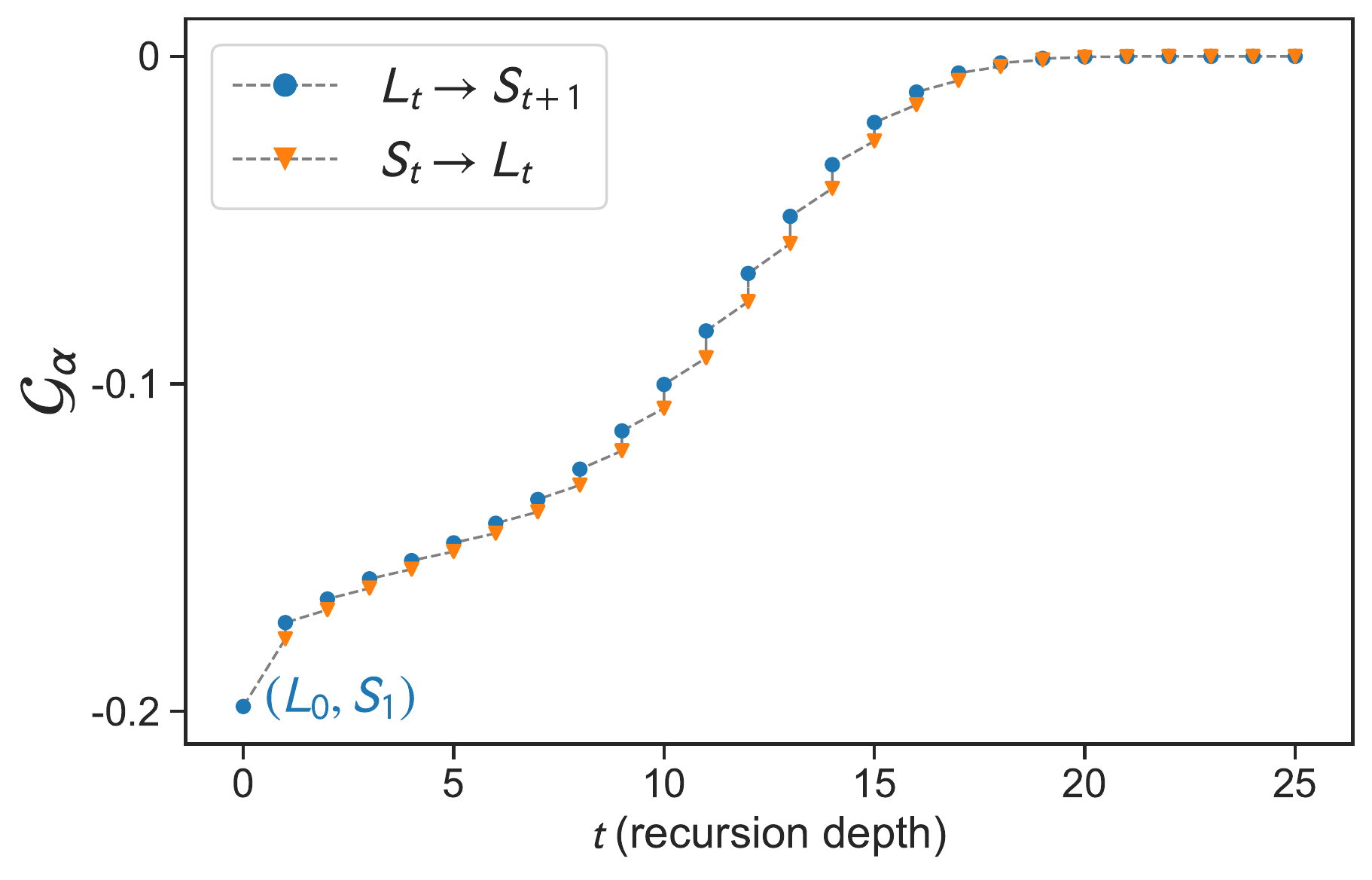}
\caption{Model simulations demonstrate that RSA recursion implements an alternating maximization algorithm. This example is based on the same model of~\figref{fig:simulations1} ($\alpha=1.2$).
The RSA tradeoff $\Ga$ improves with each speaker update (blue) and each listener update (orange).
\vspace{-.5em}
}
\label{fig:altmax-trajectory}
\end{figure}

\begin{figure*}[t!]
\centering
\subfloat[]{
    \includegraphics[height=1.8in]{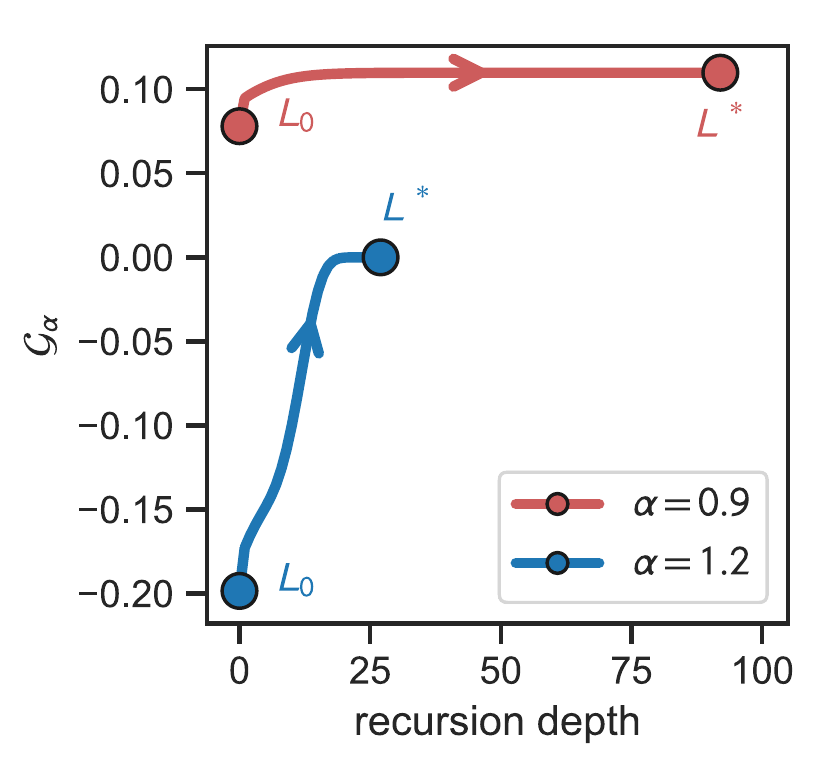}
    \label{fig:trajectory}
}\qquad
\subfloat[]{
    \includegraphics[height=1.8in]{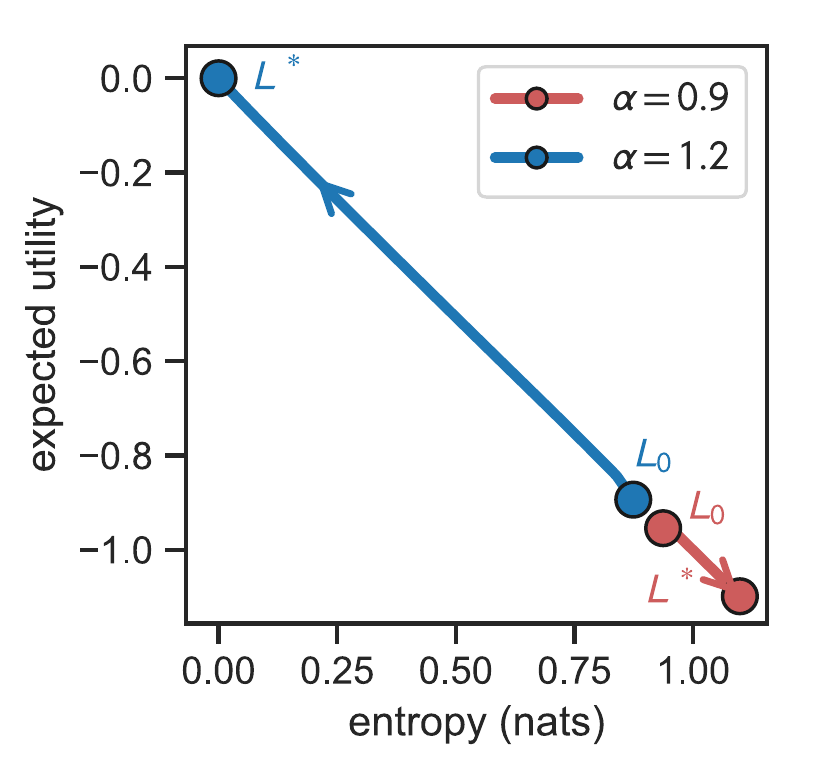}
    \label{fig:tradeoff}
}\qquad
\subfloat[]{
    {\includegraphics[width=0.18\textwidth]{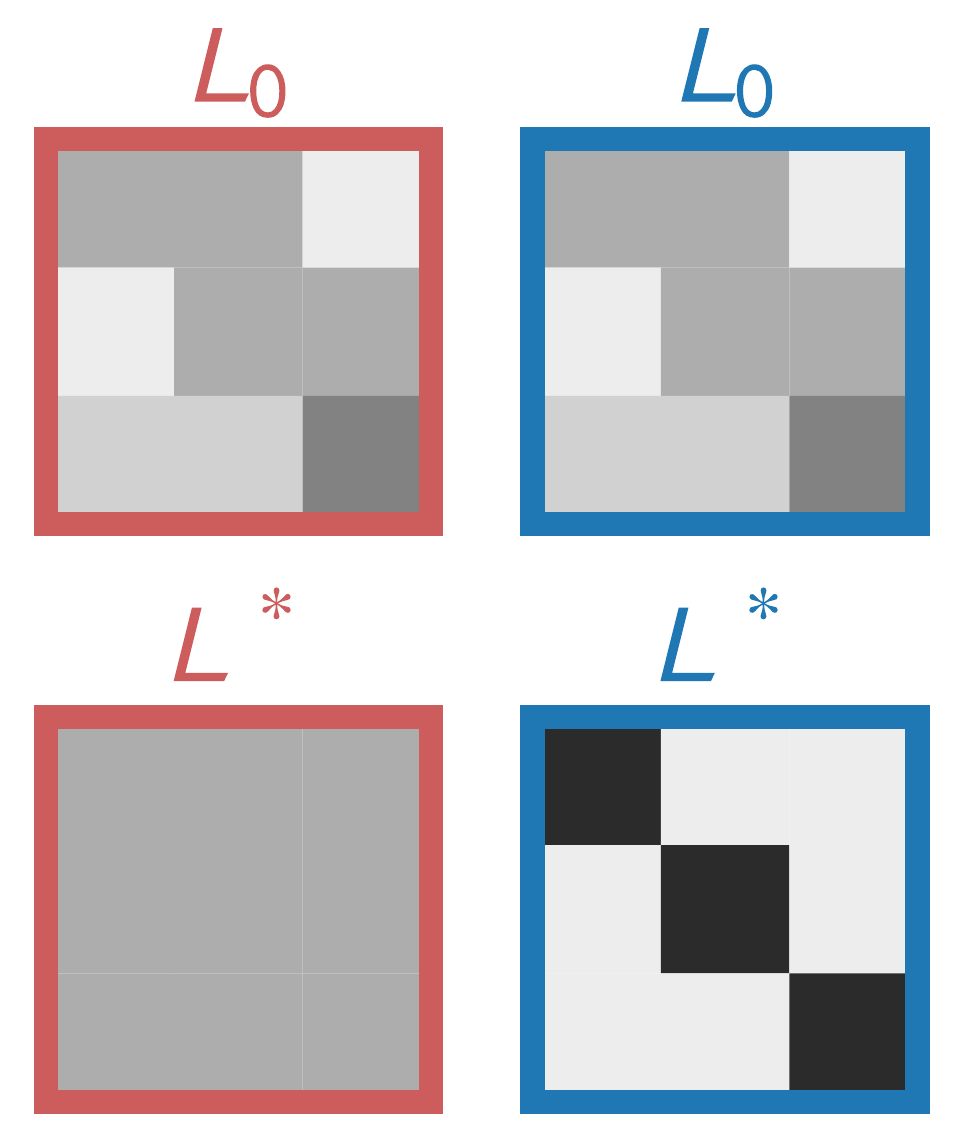}}
    \label{fig:dists}
}
\caption{Model simulations demonstrating qualitatively different behavior as a function $\alpha$. See text for details.
\textbf{(a)}~The RSA tradeoff $\Ga$ is monotonically non-decreasing with recursion depth.
\textbf{(b)}~Expected utility as a function of the speaker's entropy, $H_S(U|M)$.
The initialization of $L_0$ is the same for both values of $\alpha$; however, the initial points do not overlap because they correspond to $(L_0, S_1)$.
\textbf{(c)}~Listener distributions at initial and converged conditions for $\alpha=0.9$ (red) and $\alpha=1.2$ (blue). Matrix rows correspond to speaker utterances and columns to inferred meanings. Darker grays correspond to higher probabilities.
\vspace{-.5em}
}
\label{fig:simulations1}
\end{figure*}

First, optimizing $\Ga$ can be interpreted as a type of least-effort principle~\cite{Zipf1949}. Maximizing $\Ep{S}[V_L]$ amounts to minimizing the expected listener's surprisal and the cost of utterances. Maximizing $H_S(U|M)$ amounts to minimizing communicative effort, measured by the deviation from a random production of utterances. These two terms compete: attaining the maximal expected utility requires a very precise selection of utterances, leading to low entropy (high effort), while attaining minimal communicative effort requires a random production distribution, leading to low expected utility. 
Therefore, least-effort optimization may give rise to human pragmatic reasoning as captured by RSA.

Second, Proposition~\thmAMRSA\ in the SM implies that the RSA recursion implements an instance of the alternating maximization algorithm~\cite{Csiszar2004}. That is, given a listener $L_{t-1}$ it holds that
\eqn{
S_t = \argmax_{S}\;\Ga[S,L_{t-1}]\,,\label{eq:St-AM-RSA}
}
and given a speaker $S_t$ it holds that
\eqn{
L_t = \argmax_{L}\;\Ga[S_t,L]\,.\label{eq:Lt-AM-RSA}
}
In particular, this means that $\Ga$ is monotonically non-decreasing with recursion depth. That is, for every $t\ge1$
\eqn{
\Ga[S_t, L_{t-1}] \le \Ga[S_t, L_t] \le \Ga[S_{t+1}, L_t]\,.
}
This process is demonstrated by the model simulation of \figref{fig:altmax-trajectory}.
Because $\Ga$ is bounded from above, it follows that RSA iterations are guaranteed to converge,\footnote{To be precise, this implies that the value of $\Ga$ is guaranteed to converge to a local optimum. For stronger convergence conditions see~\cite{Csiszar1984}.}
and the fixed points of this recursion are stationary points of $\Ga$.

This result disconfirms a known conjecture about RSA dynamics.
Because the RSA speaker is guided by (soft) optimization of utterance utility, the intuition is widely held that RSA recursion locally maximizes expected utility~(e.g., \citeNP{yuan2018,peloquin-etal:2019}). However, our analysis reveals that RSA recursion optimizes $\Ga$, and this does not imply that the expected utility is maximized.
In fact, it is possible to show that the listener maximizes the expected utility, but not the speaker.
Intuitively, this holds because $H_S(U|M)$ depends only on $S$, and therefore in practice the listener's update step~\eqref{eq:Lt-AM-RSA} maximizes only the expected utility, whereas the speaker's update step~\eqref{eq:St-AM-RSA} trades it off with communicative effort, which may result in lower expected utility.
To see this more formally, fix $S_t$ and note that
\eqarrn{
\Ep{S_t}[V_t] &=& \frac1{\alpha}(\Ga[S_t,L_t] - H_{S_t}(U|M))\\
&=& \frac1{\alpha}(\max_{L}\;\Ga[S_t,L] - H_{S_t}(U|M))\\
&=& \max_{L}\;\Ep{S_t}[V_L]\,,
}
which implies in particular that $\Ep{S_t}[V_t] \ge \Ep{S_t}[V_{t-1}]$ at every recursion depth $t$. 
However, the expected utility may decrease due to the speaker update step, that is, it is possible that $\Ep{S_t}[V_t] < \Ep{S_{t-1}}[V_{t-1}]$.

This observation is exemplified in \figref{fig:simulations1}. To this end, we considered a standard RSA model with three uniformly distributed meanings and three possible utterances. The lexicon in this example is a graded\footnote{While the lexicon is often taken to be binary, graded lexica have also been notably considered (e.g., ~\citeNP{yuan2018}).} version of the lexicon shown in~\figref{fig:lexicon}, as can be seen by the structure of the literal listener (\figref{fig:dists}, $L_0$). For simplicity, we take $C(u)=0$.
Consistent with our theoretical analysis, RSA iterations always improve $\Ga$ (\figref{fig:trajectory}), but expected utility may increase (\figref{fig:tradeoff}, blue trajectory), or decrease (red trajectory), depending on $\alpha$.
Our analysis also implies that in cases where the initial state $(L_0, S_1)$ has maximal $H_S(U|M)$ given hard lexical constraints  (i.e., structural zeros in the lexicon arising when some messages do not satisfy the truth conditions of some utterances), the expected utility will not decrease.
We speculate that the possibility of RSA iteration decreasing expected utility has not previously been identified in numeric simulations because RSA initializations are typically already high in speaker conditional entropy.

In the following sections, we build on this new interpretation of RSA in order to ground it in a fundamental information-theoretic principle, and to study analytically several properties of these models, including the influence of $\alpha$ and the asymptotic behavior of their dynamics.

\section{Grounding RSA in Rate--Distortion\\ theory}

Returning to the general communication setup of RSA, the speaker can be seen as a probabilistic encoder and the listener as a probabilistic decoder. From an information-theoretic perspective, optimal encoder-decoder pairs in this setup
 are characterized by Rate--Distortion (RD) theory~\cite{Shannon1948} --- the subfield of information-theory that concerns efficient source coding with respect to a fitness function (or distortion) between a target message (speaker's meaning) and a reconstructed message (listener's interpretation). In this view, the speaker and listener should jointly optimize the tradeoff between maximizing the expected utility and minimizing the number of bits required for communication. The latter is captured by minimizing the mutual information between speaker meanings and utterances
\eqn{
I_S(M;U) = H_S(U) - H_S(U|M)\,,\label{eq:MI}
}
where $H_S(U)$ is the entropy of the marginal distribution of speaker utterances $S(u)=\sum_m P(m)S(u|m)$. In other words, from a RD perspective, the speaker and listener should jointly minimize the tradeoff
\eqn{
\Fa[S, L] =  I_S(M;U) - \alpha \Ep{S}[V_L] \label{eq:RD-RSA}\,.
}
This type of RD tradeoff is closely related to the RSA tradeoff. This can be seen by plugging equations~\eqref{eq:MI} and~\eqref{eq:MaxEnt-RSA} into~\eqref{eq:RD-RSA}, which gives
\eqn{
\Fa[S, L] =  H_S(U)- \Ga[S,L]\,.
}
We therefore refer to the optimization of $\Fa$ as RD-RSA.

The key difference between RSA and RD-RSA is the utterance entropy term $H_S(U)$.
RSA optimization, i.e., maximizing $\Ga$, is equivalent to minimizing $\Fa[S,L]$ while also maximizing $H_S(U)$. Minimizing $\Fa[S,L]$ alone yields a modified model of pragmatic reasoning. Following a similar derivation as the derivation of RSA as alternating--maximization, we show that RD-RSA predicts the following recursive reasoning process (SM, Proposition~\thmRDRSA):
\eqarrn{
S_t(u|m) &\propto& S_{t-1}(u)\exp\left(\alpha V_{t-1}(m,u)\right)\label{eq:S-RD}\\
S_t(u) &=& \sum_m S_t(u|m)P(m)  \label{eq:Su-RD} \\
L_t(m|u) &\propto& S_t(u|m)P(m)\,. \label{eq:L-RD}
}
These update equations also implement an alternating optimization algorithm, this time with respect to $\Fa$.
The optimal RD-RSA listener is Bayesian~\eqref{eq:L-RD}, exactly as in RSA. However, the optimal pragmatic speaker~\eqref{eq:S-RD} differs from the RSA speaker~\eqref{eq:S} in that it weights the soft-max utility term by the marginal utterance probability $S(u)$. One might be inclined to think that this adjustment is simply a special instance of RSA with cost function $C(u)=-\log S(u)$. We wish to emphasize that this is not the case. $S(u)$ is not pre-determined, as a cost function would be in RSA, but rather changes with each iteration as the speaker reasons about the listener.

This analysis shows that with a small adjustment, RSA can be grounded in RD theory. While RD-RSA is closely related to RSA, the theoretical motivation and precise predictions of these two principles are different.
Next, we compare several properties of RSA and RD-RSA in order to gain insight into which principle might better characterize human pragmatic reasoning.

\section{Properties of RSA and RD-RSA}

Building on the results presented above, we study the asymptotic tendencies of RSA and RD-RSA and evaluate their predictions on existing experimental data. 

\subsection*{Asymptotic behavior and criticality of $\alpha=1$}

To gain a better understanding of the dynamics of RSA and RD-RSA recursion, we analyze their asymptotic behavior. That is, we focus on the set of optimal pairs $(S^*,L^*)$ as a function of $\alpha$. For RSA, these are the fixed points of equations~\eqref{eq:St-AM-RSA} and~\eqref{eq:Lt-AM-RSA} (or equivalently, \eqref{eq:S} and~\eqref{eq:L}), and for \mbox{RD-RSA} these are the fixed points of equations~\eqref{eq:S-RD}-\eqref{eq:L-RD}. This reveals the general tendencies of the model dynamics, and shows how the value of $\alpha$ may influence the model predictions.
Therefore, this asymptotic analysis is useful even if human pragmatic reasoning, to the extent that it resembles this recursive process, might be confined to a small number of iterations. As before, the formal proofs are provided in the SM (Section 3), and here we discuss these results while focusing on the main conclusions.

We begin by considering a basic RSA setup in which the speaker's meanings are distributed uniformly, $C(u)=0$, and the number of unique utterances is the same as the number of meanings. In addition, we first assume a graded lexicon $l(m,u)\in(0,1]$, where the values defining the applicability of utterances to meanings can be arbitrarily small but not zero.
We prove in the SM that in both RSA and \mbox{RD-RSA}, $\alpha=1$ is a critical point at which the optimization dynamics changes its direction. That is, there are two regimes: $\alpha\in[0,1]$, in which the non-informative solution (i.e., maximal $H_S(U|M)$ in RSA, and $I_S(M;U)=0$ in RD-RSA) is optimal; and $\alpha\ge1$, in which the maximal-utility solution (i.e., \mbox{$\Ep{S}[V_L]=0$}) is optimal.
This transition can be clearly seen for RSA in~\Cref{fig:tradeoff}, and similar simulations of \mbox{RD-RSA} exhibit the same behavior (not shown). Our analysis in the SM also shows that RSA and RD-RSA differ at $\alpha=1$. At this point, in RD-RSA, but not in RSA, all fixed points are globally optimal.

Next, we consider the same basic setup but allow structural zeros in the lexicon.
In this case, if there exists a maximal-utility solution (e.g., a bijection from meanings to utterances) that does not violate the lexicon, which is often the case, then
it remains a global optimum for $\alpha\ge1$, as in the case of a graded lexicon.
On the other hand, the non-informative solution typically violates the lexicon, leading to a different behavior in the regime of $\alpha<1$.
We demonstrate this by the model simulations shown in \Cref{fig:vogel} (top), using the binary lexicon of \Cref{fig:lexicon}.
As expected, for $\alpha>1$ both RSA and RD-RSA converge to the maximal-utility point (blue and green trajectories), and for $\alpha=1$ RD-RSA converges immediately, while RSA still converges to the maximal-utility point (purple trajectories). For $\alpha <1$, the models cannot reach the non-informative solution
because it violates the lexicon (red trajectories). In this case, RD-RSA's trajectory moves toward the non-informative solution but converges at a solution with $I(M;U)>0$. RSA's trajectory moves in the other direction, but importantly it converges at a solution with lower expected utility compared to those of $\alpha\ge1$.

\begin{figure}[t!]
\centering
\includegraphics[width=.95\linewidth]{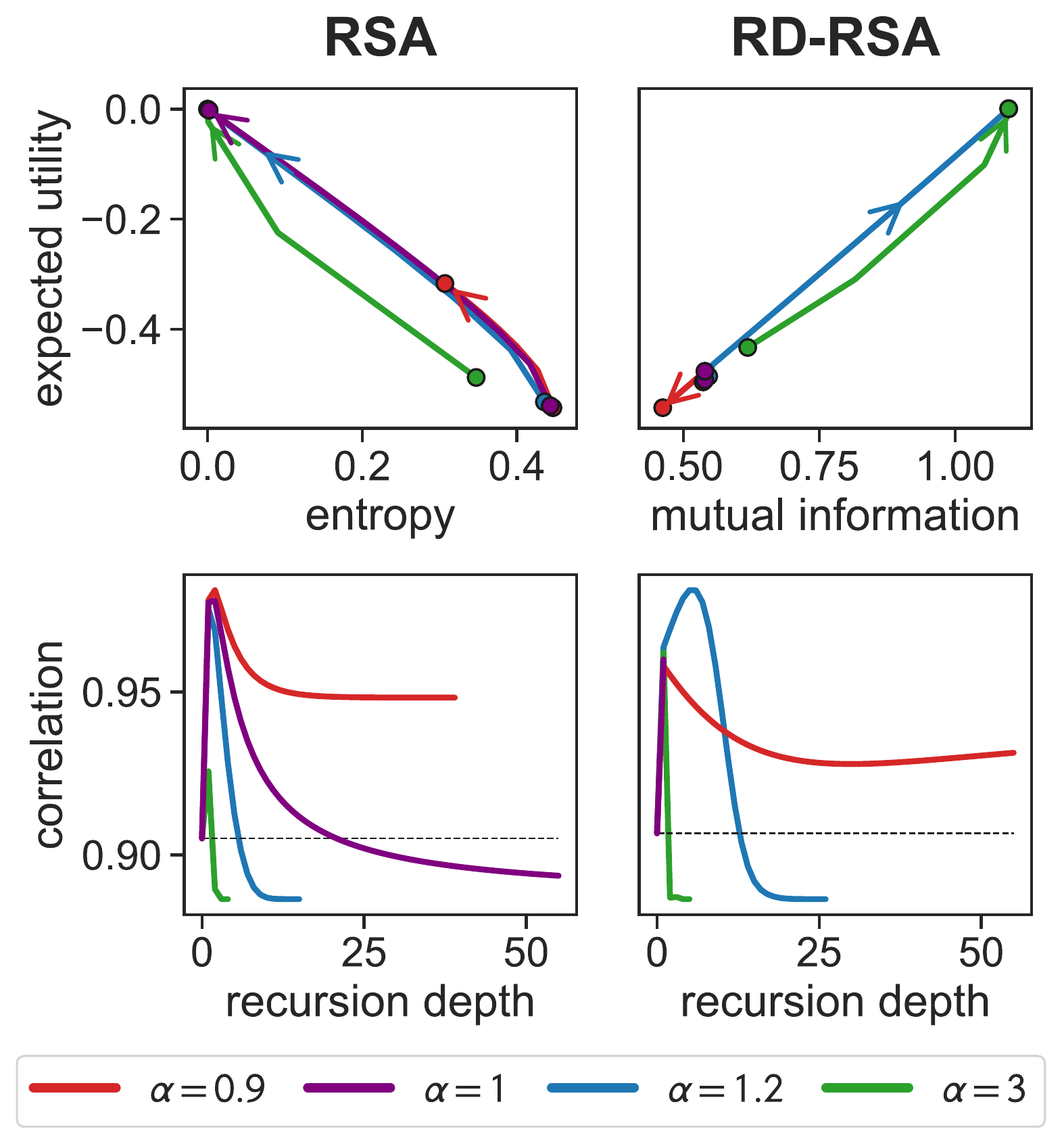}
\caption{Model simulations and evaluation with respect to the data from \protect\citeA{Vogel2014}.
Top: Simulated model trajectories from $L_0$ to $L^*$ in RSA (left) and RD-RSA (right). Bottom: Pearson correlation~$\rho$ between model predictions and human behavioral results as a function of recursion depth. Horizontal dashed line corresponds to the literal listener $L_0$.
}
\label{fig:vogel}
\end{figure}

Finally, we note that considering a non-trivial cost $C(u)$, or a two-positional cost $C(m,u)$, may change substantially the dynamics of the models. We present in the SM a preliminary analysis of the influence of the cost function on the behavior of the models, and leave to future work a more comprehensive analysis of this case.

\subsection*{Comparison with human behavior}

We have seen thus far that even though RD-RSA is closely related to RSA, it may generate different predictions. This raises the question how well can RD-RSA account for human behavior compared to RSA.
To address this, we consider data from an online reference game experiment conducted by \citeA{Vogel2014}.
Each trial presents a set of target objects and a set of possible speaker utterances\footnote{Several different stimulus types were presented during the experiment, but the structure of the lexicon was consistent across trials. Note that we only consider data from the Complex condition.} that conform to the lexicon shown in \Cref{fig:lexicon}. Participants were then given a speaker utterance, and were asked to indicate which of the target objects they think the speaker refers to.
As explained in~\Cref{fig:rsa-setup}, this experimental setup invites a complex pattern of pragmatic reasoning.
We estimated an empirical human listener from the responses recorded by  Vogel et al., and compared that to the pragmatic listeners predicted by RSA and RD-RSA with the lexicon of \Cref{fig:lexicon}.
 Following~\citeA{Frank2018}, who have previously presented an RSA account of these data, we assume a uniform prior over meanings and no utterance cost. This setting corresponds to the model simulations of \Cref{fig:vogel} (top), discussed earlier.
\Cref{fig:vogel} (bottom) shows the correlation between the model predictions and the behavioral data as a function of recursion depth. It can be seen that the predictions of both RSA and RD-RSA improve in the first few iterations, and then deteriorate with depth (except for $\alpha=1$ in RD-RSA, which converges immediately). This is consistent with prior work on RSA (e.g.,~\citeNP{Frank2018}). While the value of $\alpha$ and depth that best fit the data differ between RSA ($\alpha=0.9$, depth 1) and RD-RSA ($\alpha=1.2$, depth 5), their correlation is similar (RSA: $\rho=0.98$, RD-RSA: $\rho=0.97$), and so are their predicted listeners (see SM Section 4).
Therefore, RD-RSA is comparable to RSA in its ability to account for human behavior in this task.

\begin{figure}[b!]
\centering
\includegraphics[width=.95\linewidth]{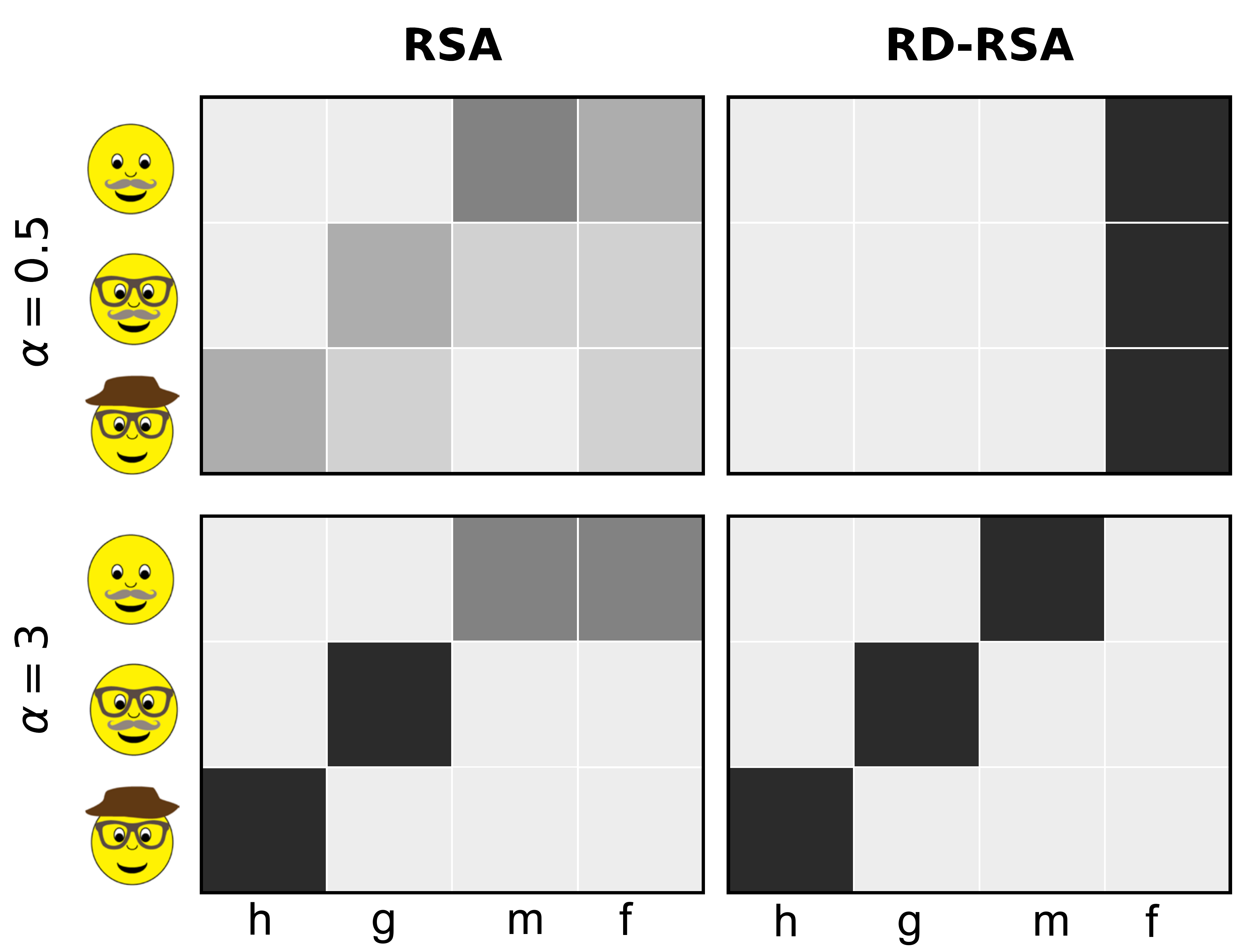}
\caption{Optimal speaker distributions for $\alpha=0.5$ (top) and $\alpha=3$ (bottom). Here, we augmented the lexicon of \Cref{fig:lexicon} with an extra utterance that can be applied to all referents  (e.g.\ \emph{friend}, labeled as `f').
RSA predicts a bias toward randomized productions, while RD-RSA does not.
\vspace{-.5em}
}
\label{fig:rsa-bias}
\end{figure}

\subsection*{Conditional entropy or mutual information?}

Last, we demonstrate an important implication of the key difference between RSA and RD-RSA, that is, maximizing $H_S(U|M)$ rather than minimizing $I_S(M;U)$. As noted before, this difference boils down to RSA maximizing $H_S(U)$ while also optimizing the RD-RSA objective. This implies that the RSA speaker, as opposed to the RD-RSA speaker, is biased toward random utterance productions. To demonstrate this, we ran similar simulations as in~\Cref{fig:vogel}, but now with an extra utterance that can be applied to all meanings (e.g., \emph{friend}). \Cref{fig:rsa-bias} shows the optimal speakers in this case for two values of $\alpha$.
For $\alpha=0.5$ there is strong pressure to minimize communicative effort. In this case, the RSA speaker uses all utterances almost uniformly, which does not convey much information to the listener, while the RD-RSA speaker follows an intuitively simpler production distribution that exclusively uses the extra utterance. For $\alpha=3$ there is strong pressure to maximize expected utility. In this case, the RSA speaker uses the extra utterance to randomize the description of one referent, because this increases entropy without changing the expected utility. In RD-RSA this speaker is also optimal, but so is the speaker that uses a single unique utterance for each referent (shown in the bottom right of~\figref{fig:rsa-bias}). Therefore, RD-RSA does not predict a cognitive bias toward randomness.

\section{Discussion}

Pragmatic reasoning is a crucial aspect of human language, often understood within the Rational Speech Act (RSA) framework. While this framework has been remarkably successful in explaining a wide range of psycholinguisic phenomena, it has not been cast in terms of a general optimization principle, leaving open the question of whether such a principle exists for characterizing human pragmatic behavior. Here, we have addressed this open question by presenting a novel information-theoretic analysis of the RSA framework. We have shown that RSA's recursive reasoning can be derived from least-effort optimization, in contrast to a widely held view of RSA as implementing a heuristic process for maximizing utility. Furthermore, we have shown that with a small adjustment, RSA can be grounded in a more fundamental optimization principle, RD-RSA, which is based on Shannon's Rate--Distortion (RD) theory.

We believe that RD-RSA is particularly noteworthy for several reasons.
First, as our results suggest, RD-RSA avoids RSA's bias toward non-informative random productions, while maintaining similar ability to account for human behavior. Our analysis has been based only on data from~\citeA{Vogel2014}, and so an important direction for future work is to further test the predictions of RD-RSA on additional experimental data.

Second, RD-RSA not only suggests a general information-theoretic principle that may give rise to human pragmatic reasoning, but also provides a potential theoretical link between pragmatics and several other aspects of language~\cite{Zaslavsky2018,Hahn2020} and cognition~\cite{Sims2016,Sims2018,Gershman2020} to which RD theory has recently been applied successfully. We also note that RSA's optimization principle~\eqref{eq:MaxEnt-RSA} suggests interesting theoretical links to other related frameworks, such as the recent Optimal Transport approach to cooperative communication~\cite{Wang2019}, although the latter applies only to a single iteration of pragmatic reasoning in the RSA framework (but see~\citeNP{yuan2018}, for the special case of $\alpha=1$). 

Finally, we argue that RD-RSA addresses a major concern about the applicability of information theory to pragmatics. As noted by Sperber and Wilson: ``there is a gap between the semantic representations of sentences and the thoughts actually communicated by utterances. This gap is filled not by more coding, but by inference'' (\citeNP{Sperber1986}, p.~9). RD theory lies at the intersection of information theory and statistical inference, and thus it may capture both aspects of coding and inference in pragmatics, and perhaps in language more generally.

\section*{Acknowledgments}
N.Z. was supported by a BCS Fellowship in Computation; J.H. was supported by the NIH under award number T32NS105587 and an NSF Graduate Research Fellowship; R.P.L. was supported by NSF grant BCS-1456081, a Google Faculty Research Award, and Elemental Cognition.

\bibliographystyle{apacite}
\setlength{\bibleftmargin}{.125in}
\setlength{\bibindent}{-\bibleftmargin}
\bibliography{bibliography}
\protect\balance

\newpage
\newgeometry{margin=2.5cm}
\fontsize{11}{14}\selectfont
\setcounter{section}{0}
\setcounter{footnote}{0}
\onecolumn

\vspace*{.5cm}
\begin{center}
{\LARGE Supplementary Material\\[.5em] \textbf{A Rate--Distortion view of human pragmatic reasoning}}
\end{center}
\vspace*{.5cm}

\section{Understanding RSA dynamics as Alternating Maximization}

Here we prove that RSA recursion implements an alternating maximization algorithm for optimizing
\eqn{
\Ga[S, L] = H_S(U|M) +\alpha\Ep{S}[V_L]\,.\label{sm-eq:MaxEnt-RSA}
}
Before doing so, we introduce several required definitions and notations.
First, we formally define an RSA reference game by a tuple $<\M, \U, P(m), l, C>$, where $\M$ is a finite set of meanings (referents), $\U$ is a finite set of utterances, $P(m)$ is a prior distribution over $\M$, $l$ is a lexicon, and $C$ is a cost function. We formally define a lexicon by a mapping $l:\M\times\U\rightarrow [0,1]$, such that $l(m,u) > 0$ if $u$ can be applied to $m$ and $l(m,u) = 0$ otherwise. Unless stated otherwise, we assume that $C:\U\rightarrow\reals_+$, namely that $C$ is a non-negative utterance cost function. More generally, $C$ could also be a two-positional cost function, i.e., $C:\M\times\U\rightarrow\reals_+$ .
Next, we define the set of all speaker and listener distributions that do not violate a given lexicon $l$. Denote by $\simplex{\U}$ the simplex of probability distributions over $\U$, and by $\simplex{\U}^{\M}$ the set of all conditional probability distributions of $U$ given $M$. Similarly, denote by $\simplex{\M}^{\U}$ the set of all conditional probability distributions of $M$ given $U$. The set of all possible speakers that do not violate the lexicon is then
\eqn{
\S_l = \{S\in \simplex{\U}^{\M}: S(u|m) = 0\, \mbox{ if } l(m,u)=0 \}\,,
}
and the set of all possible listeners that do not violate the lexicon is
\eqn{
\L_l = \{L\in \simplex{\M}^{\U}: L(m|u) = 0\, \mbox{ if } l(m,u)=0 \}\,.
}
It is easy to verify that $\S$ and $\L$ are convex sets.

\begin{Proposition}[RSA optimization]
\label{thm:maxent-rsa}
Let $\alpha\ge0$. The following statements hold for RSA:
\begin{itemize}
\item RSA recursion implements an alternating maximization: for all $t\ge1$, for a fixed $L_{t-1}$ it holds that 
\eqn{
S_t = \argmax_{S\in\simplex{\U}^{\M}}\;\Ga[S,L_{t-1}]\,,\label{sm-eq:S-AM}
}
and for a fixed $S_t$ it holds that
\eqn{
L_t = \argmax_{L\in\simplex{\M}^{\U}}\;\Ga[S_t,L]\,,\label{sm-eq:L-AM}
}
where $S_t$ and $L_t$ are RSA's speaker and listener distributions at recursion depth $t$.
\item If $L_{t-1}\in\L_l$ then $S_t\in\S_l$, and if $S_t\in\S_l$ then $L_t\in\L_l$. That is, RSA iterations do not violate the hard lexicon constraints. 
\item The fixed points of the RSA recursion are stationary points of $\Ga$.
\end{itemize}
\end{Proposition}

\begin{proof}
First, fix $L_{t-1}$ and note that the function $g(S)=\Ga[S, L_{t-1}]$ is concave in $S$. To find a maximizer for $g(S)$ over $\simplex{\U}^{\M}$ we define the Lagrangian 
\eq{
\mathbb{L}[S;\lambda] = g(S)- \sum_m \lambda(m) \sum_u S(u|m)\,,
}
where $\lambda(m)$ are the normalization Lagrange multipliers.\footnote{We omit the non-negativity constraints because these constraints are inactive.}
Note that if for some $m$ and $u$ it holds that \mbox{$L_{t-1}(m|u)=0$} and $S(u|m)>0$, then $g(S) = -\infty$. Therefore, at the maximum, it necessarily holds that if $L_{t-1}(m|u)=0$ then also $S(m|u)=0$ (following the convention that $0\log0=0$). In particular, this implies that if $L_{t-1}\in\L_l$ then $\argmax g(S)\in\S_l$. That is,  if $L_{t-1}$ does not violate the lexicon, then maximizing $g(S)$ is guaranteed to give a speaker that also does not violate the lexicon.
Taking the derivative of $\mathbb{L}[S;\lambda]$ with respect to $S(u|m)$, for every $m$ and $u$ such that $L_{t-1}(m|u)>0$, gives
\eq{
\dpp[\mathbb{L}]{S(u|m)} = P(m)\left[ -\log S(u|m) -1 +\alpha V_{t-1}(m,u) \right] - \lambda(m)\,.
}
Equating these derivatives to zero gives RSA's speaker $S_t$ (equation (2) in the main text), as a necessary condition for optimality. Because $g(S)$ is concave, this is also a sufficient condition for this step.

Next, fix $S_t$ and consider the function $h(L)=\Ga[S_{t}, L]$. This function is concave in $L$. To find a maximizer for $h$ over $\simplex{\M}^{\U}$, we define as before the corresponding Lagrangian and take its derivative with respect to $L(m|u)$. This gives
\eq{
\dpp[\mathbb{L}]{L(m|u)} = \alpha P(m)S_t(u|m) \frac{1}{L(m|u)} - \lambda(u)\,.
}
Equating this derivative to zero gives RSA's Bayesian listener $L_t$ (equation (3) in the main text) as a necessary condition for optimality. Because $h(L)$ is concave, this is also a sufficient condition for this step. It is also easy to verify that if $S_t(u|m)=0$ then $L_t(m|u)=0$, and therefore if $S_t\in\S_l$ then $L_t\in\L_l$.

Finally, at a fixed point $(S^*, L^*)$ both the derivatives with respect to $S$ and  to $L$ are zero. Because these are also the derivatives of $\Ga$ over $\simplex{\U}^{\M}\times\simplex{\M}^{\U}$, it holds that $(S^*, L^*)$ is a stationary point of $\Ga$. Note that $\Ga$ is not jointly concave in $S$ and in $L$, and therefore, $(S^*, L^*)$ is not necessarily a global maximum.
\end{proof}

\section{Derivation of RD-RSA}

In this section we derive the RD-RSA update equations from the minimization of
\eqn{
\Fa[S,L] =  I_S(M;U) - \alpha \Ep{S}[V_L]\,.
}
This can be seen as a type of Rate--Distortion (RD) optimization problem with a \emph{variable} distortion measure \mbox{$d(m,u)=-V_L(m,u)$} between meanings and utterances.

\begin{Proposition}[RD-RSA]
\label{thm:basic-rd-rsa-opt}
Let $S\in\simplex{\U}^{\M}$ and $L\in\simplex{\M}^{\U}$. Given $\alpha>0$, $S$ and $L$ are stationary points of $\Fa$ if and only if they satisfy the following self-consistent conditions:
\eqarrn{
S(u|m) &\propto& S(u)\exp\left(\alpha V_L(m,u)\right)\label{sm-eq:S-RD}\\
S(u) &=& \sum_m S(u|m)P(m)\\
L(m|u) &=& \frac{S(u|m)P(m)}{S(u)}\label{sm-eq:L-RD}
}
\end{Proposition}

\begin{proof}
The main idea of the proof is to take the derivatives of $\Fa$ w.r.t. $S(u|m)$, $S(u)$, and $L(m|u)$, and equate these derivatives to zero, which gives the RD-RSA equations \eqref{sm-eq:S-RD}-\eqref{sm-eq:L-RD} as necessary conditions for optimality.
This derivation is similar to the derivation in the proof of Proposition~\ref{thm:maxent-rsa}, and therefore we do not repeat it here.
\end{proof}

Note that $\Fa$ is convex in $S(u|m)$, $S(u)$, and $L(m|u)$, although it is not jointly convex in these three distributions. Therefore, similar to the proof of Proposition~\ref{thm:maxent-rsa}, it holds that $\Fa$ can be optimized via an alternating minimization algorithm that iteratively updates equations \eqref{sm-eq:S-RD}-\eqref{sm-eq:L-RD}, as described in the main text. However, because $\Fa$ is not jointly convex in these variables, this iterative algorithm will not necessarily converge to a global optimum.

\section{Asymptotic behavior and the criticality of $\alpha=1$}

In this section we analyze the asymptotic behavior of RSA and RD-RSA dynamics. In both cases, we focus mainly on the basic RSA setup discussed in the main text. We also present preliminary analysis of the influence of the cost function.

\subsection{RSA}

Denote by $\Ga^*$ the maximal value of $\Ga$ given $\alpha$, and let $\Sa$ and $\La$ be optimal speaker and listener distributions that attain $\Ga^*$. That is, $\Ga^*=\Ga[\Sa,\La]$. The following proposition characterizes $\Ga^*$, $\Sa$, and $\La$, as a function of $\alpha$, in a basic RSA setup.

\begin{Proposition}[Asymptotic behavior of RSA]
\label{thm:basic-rsa-asympt}
Let $C(u)$ be a constant function, $P(m)$ be the uniform distribution over $\M$, and assume $K=|\M|=|\U|$. In addition, assume a graded lexicon $l$ with no structural zeros. Then the following statements hold:
\begin{enumerate}
\item $\Ga^* = \max\{(1-\alpha)\log K, 0\}$
\item For $\alpha\in [0,1]$, $\Sa(u|m) = \frac{1}{|\U|}$ and $\La(m|u)=P(m)$.
\item For $\alpha\ge1$, 
$\Sa$ and $\La$ are deterministic distributions defined by a bijection from $\M$ to $\U$.
\end{enumerate}
\end{Proposition}

\begin{proof}
We prove these claims by first deriving an upper bound on $\Ga$ and then showing that the given $\Sa$ and $\La$ attain this bound in the two regimes of $\alpha$. Assume w.l.o.g. that $C(u)=0$, and let $S(m|u)$ be the posterior distribution with respect to $S(u|m)$ and $P(m)$. 
For any $S$ and $L$ it holds that
\eqarrn{
\Ga[S,L] & \le & H_S(U|M) + \alpha\Ep{S}\left[\log S(m|u)\right]\label{sm-eq:kl-bound}\\
& = & H_S(U|M) + \alpha\Ep{S}\left[\log \frac{S(u|m)P(m)}{S(u)}\right]\label{sm-eq:sub-bayes}\\
& = & (\alpha-1) I_S(M;U) + H_S(U) - \alpha H(M)\,.\label{sm-eq:G_bound}
}
Equation \eqref{sm-eq:kl-bound} follows from the fact that $\sum_x p(x)\log q(x) \le \sum_x p(x)\log p(x)$ for any two distributions, and specifically for $S(m|u)$ and $L(m|u)$. Equation \eqref{sm-eq:sub-bayes} follows from substituting Bayes' rule, and \eqref{sm-eq:G_bound} from the definition of entropy and the identity $I_S(M;U) = H_S(U)-H_S(U|M)$.

For $\alpha\in[0,1]$, this bound is maximal when $I_S(M;U)=0$ and $H_S(U)= \log K$. Therefore, in this regime $\Ga[S,L]\le\log K - \alpha H(M)$, 
and it is easy to verify that $S(u|m)=\frac1K$ and $L(m|u)=P(m)$ attain this bound (simply substitute these distributions in the definition of $\Ga$).
Specifically, when $P(m)$ is uniform, it holds that $H(M)= \log |\M| = \log K$, and therefore $\Ga[S,L]\le (1-\alpha)\log K $.
For $\alpha \ge1$ it holds that $(\alpha-1) I_S(M;U) \le (\alpha-1) \log K$, and therefore
$\Ga[S,L] \le \alpha \log K - \alpha H(M)$.
When $P(m)$ is uniform, this bound becomes $\Ga[S,L] \le 0$.
Let $\phi:\M\rightarrow\U$ be a bijection, and set $S(u|m)=\delta_{u, \phi(m)}$ and $L(m|u)= S(m|u) = \delta_{u, \phi(m)}$. In this case, $H_S(U|M)=0$ and $\E_S[V_L]=0$, following the convention that $0\log 0 =0$. Therefore, these distributions attain the bound for $\alpha\ge1$.
Putting everything together gives $\Ga^* = \max\{(1-\alpha)\log K, 0\}$, which concludes the proof.
\end{proof}

Proposition~\ref{thm:basic-rsa-asympt} shows that in the basic RSA setup that corresponds to its assumptions, there is only one critical value $\alpha_c=1$, which determines the global optimum of $\Ga$ and the asymptotic tendency of the RSA dynamics.
However, when $C(u)$ is not a constant function there could be multiple critical values $\alpha_c$. We next show that in this case the first critical value $\alpha_0$, at which the non-informative solution $L(m|u)=P(m)$ looses its optimality, is $\alpha_0\ge1$.
To see this, notice that adding the utterance cost to the bound in~\eqref{sm-eq:G_bound} gives
\eqarrn{
\Ga[S,L] &\le&(\alpha-1) I_S(M;U) + H_S(U) - \alpha H(M) -\alpha\Ep{S}[C(U)]\\
&=& (\alpha-1) I_S(M;U) - D[S(u)\|\Qa(u)] + \log\Za -\alpha H(M) \label{sm-eq:G_bound_w_cost}
}
where $\Qa(u)$ is the maximum entropy distribution over $\U$ with respect to $C(u)$ defined by
\eq{
\Qa(u) = \frac{e^{-\alpha C(u)}}{\Za}\;, \;\;\;\,\Za=\sum_{u}e^{-\alpha C(u)}\,
}
and $D[\cdot\|\cdot]$ is the Kullback-Leibler (KL) divergence. For $\alpha\le1$, the first two terms in~\eqref{sm-eq:G_bound_w_cost} are non-positive and therefore~\eqref{sm-eq:G_bound_w_cost} can be further bounded from above, yielding
\eqn{
\Ga[S,L] \le \log\Za -\alpha H(M)\,. \label{sm-eq:G_bound_w_cost_U}
}
It is easy to verify that this upper bound is attained by $\Sa(u|m) = Q_{\alpha}(u)$ and $\La(m|u)=P(m)$ (as before, to see this substitute these distribution in $\Ga$). In this case, $\Sa(u|m)$ changes continuously for $\alpha\le\alpha_0$, even though these changes do not convey any information to the listener.
 In other words, in this regime, the RSA model predicts that a pragmatic speaker will not try to convey any information to the listener ($I_S(M;U)=0$), but will rather seek the minimal deviation from random utterance production that reduces the expected utterance cost $\Ep{S}[C(U)]$ to a tolerable degree, determined by $\alpha$. This is another demonstration of RSA's bias toward random utterance production.
Finally, we note that more generally, if the cost function is two-positional, that is $C(m,u)$, then it is possible that $\alpha_0<1$.

\subsection{RD-RSA}

Next, we characterize the asymptotic behavior of RD-RSA in the basic setup discussed in the main text. Denote by $\Fa^*$ the minimal value of $\Fa$ given $\alpha$, and let $\Sa$ and $\La$ be optimal speaker and listener distributions that attain $\Fa^*$.

\begin{Proposition}
\label{thm:basic-rd-rsa-asympt}
Let $C(u)$ be a constant function, then the following statements hold for RD-RSA:
\begin{enumerate}

\item For $\alpha\in [0,1)$,
\eq{
\Sa \in \argmin_S I_S(M;U)
}
\item For $\alpha > 1$, 
\eq{
\Sa \in \argmax_S I_S(M;U)
}
\item For $\alpha=1$, and all stationary points are optimal.
\end{enumerate}
\end{Proposition}

\begin{proof}
The idea of the proof is similar to the proof of Proposition~\ref{thm:basic-rsa-asympt}. Here, however, we derive a lower bound for $\Fa$.
For this, we take similar steps as in \eqref{sm-eq:kl-bound}-\eqref{sm-eq:sub-bayes} but adapt them to $\Fa$ by replacing the conditional entropy in the first term by $I_S(M;U)$ and changing the sign of the second term. This gives the lower bound
\eqarrn{
\Fa[S,L] & \ge &  (1-\alpha) I_S(M;U) + \alpha H(M)\,.\label{sm-eq:F_bound}
}
For $\alpha \in [0,1)$, this lower bound is minimal when $I_S(M;U)$ is minimal, i.e. when $I_S(M;U)=0$, which is attained by a non-informative speaker, e.g. $\Sa(u|m)=\frac1{|\U|}$.
For $\alpha>1$, this lower bound is minimal when $I_S(M;U)$ is \emph{maximal}. Therefore, the optimum in this regime is given by $\argmax_S I_S(M;U)$.
If there exists a bijection $\phi:\M\rightarrow\U$ that does not violate the lexicon, then $\Sa(u|m)=\delta_{u, \phi(m)}$ and $\La(m|u)= \Sa(m|u) = \delta_{u, \phi(m)}$ attain this bound.
Finally, for $\alpha=1$, any fixed point $(S^*, L^*)$ of the RD-RSA equations gives
\eqn{
\Ep{S^*}\left[\log L^*(m|u)\right] = \Ep{S^*}\left[\log S^*(m|u)\right] = -H_S^*(M|U)\,,
}
and therefore $\Fa[S^*,L^*] = I_{S^*}(M;U) + H_S^*(M|U)= H(M)$.
This means that all fixed points are equally good in this regime. Note also that the lower bound~\eqref{sm-eq:F_bound} in this case becomes $H(M)$.
\end{proof}

\section{Comparison with human behavior}

In the main text we have shown that both RSA and RD-RSA produce listener distributions that are highly correlated with the empirical human listener estimated from the experimental data of Vogel et al. (2014).  Here we supplement that evaluation with the figure below, which shows that the best RSA listener and the best RD-RSA listener are indeed very similar to each other and to the empirically estimated human listener.

\begin{figure}[h!]
\centering
\includegraphics[width=.5\linewidth]{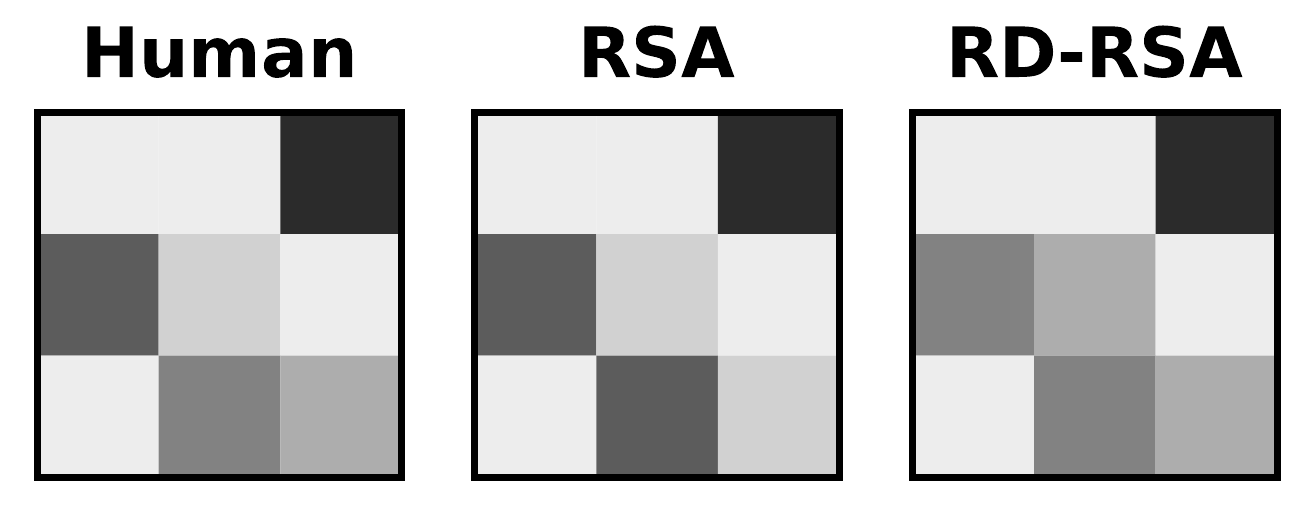}
\caption{Left: Human listener distribution estimated from the data of Vogel et al. (2014). Middle: RSA's listener distribution for $\alpha=0.9$ and recursion depth 1. Right: RD-RSA's listener distribution for $\alpha=1.2$ and recursion depth 5.}
\label{fig:vogel-sm}
\end{figure}

\end{document}